\documentclass[a4paper,10pt]{article}
\usepackage[T1]{fontenc}
\usepackage{parskip}
\usepackage[font={small}]{caption}
\usepackage[hidelinks]{hyperref}
\usepackage[margin=2.5cm]{geometry}
\usepackage{times}
\usepackage{amsmath,amssymb,amsthm}
\usepackage[affil-it]{authblk}
\usepackage{listings}
\usepackage[square,numbers]{natbib}

\usepackage{graphicx}          

\usepackage{array,booktabs} 			
\usepackage{amsmath,amsfonts,amssymb,bm} 
\usepackage{graphicx}
\usepackage{hyperref}
\usepackage{color}
\usepackage{soul}
\usepackage{float}
\usepackage{placeins}
\usepackage{siunitx}[=v2]

\usepackage{algorithm}
\usepackage{algorithmic}
\usepackage{etoolbox}

\newcommand{\uniformnormal}[2]{\ensuremath{\mathcal{U}(#1,#2)}}
\newcommand{\thetamodel}{{\phi}}

\newcommand{\yAll}{\mathbf{y}}
\newcommand{\yAllExpand}{\left\{y_1, \dots, y_T\right\}}
\newcommand{\uAll}{\mathbf{u}}
\newcommand{\uAllExpand}{\left\{u_1, \dots, u_T\right\}}
\newcommand{\xAll}{\mathbf{x}}
\newcommand{\xAllExpand}{\left\{x_1, \dots, x_{T+1}\right\}}
\newcommand{\Transp}{\mathsf{T}}
\newcommand{\pn}{v} 
\newcommand{\mn}{e} 

\newcommand{\pmat}[1]{\begin{bmatrix}#1 \end{bmatrix}}
\newcommand{\bmat}[1]{\begin{bmatrix}#1\end{bmatrix}}

\newtheorem{lemma}{Lemma}

\hypersetup{
	colorlinks=true,       
	linkcolor=blue,          
	citecolor=blue,        
	urlcolor=blue           
}

\date{\today}

\title{Variational System Identification for Nonlinear\\ State-Space Models} 

\author[1]{Jarrad Courts\thanks{\url{jarrad.courts@uon.edu.au}}}
\author[1]{Adrian Wills\thanks{\url{Adrian.Wills@newcastle.edu.au}}}
\author[2]{Thomas B. Sch\"on\thanks{\url{thomas.schon@it.uu.se}}, \thanks{This research was financially supported by the Swedish Foundation for Strategic Research (SSF) via the project \emph{ASSEMBLE} (contract number: RIT15-0012), by the Swedish Research Council via the projects \emph{Deep probabilistic regression -- new models and learning algorithms} (contract number: 2021-04301) and  \emph{NewLEADS - New Directions in Learning Dynamical Systems} (contract number: 621-2016-06079) and by \emph{Kjell och M{\"a}rta Beijer Foundation}.}}
\author[1]{Brett Ninness\thanks{\url{brett.ninness@newcastle.edu.au}}}

\affil[1]{University of Newcastle, School of Engineering, Australia}        
\affil[2]{Department of Information Technology, Uppsala University, Uppsala, Sweden}

\begin{document}
		
	\maketitle
	

\begin{abstract}                          
  This paper considers parameter estimation for nonlinear state-space models, which is an important but challenging problem. We address this challenge by employing a variational inference (VI) approach, which is a principled method that has deep connections to maximum likelihood estimation. This VI approach ultimately provides estimates of the model as solutions to an optimisation problem, which is deterministic, tractable and can be solved using standard optimisation tools. A specialisation of this approach for systems with additive Gaussian noise is also detailed. The proposed method is examined numerically on a range of simulated and real examples focusing on the robustness to parameter initialisation; additionally, favourable comparisons are performed against state-of-the-art alternatives.
\end{abstract}


\section{Introduction}
	The problem of system identification is long-standing and has significant practical applications with a large body of related literature across many different fields \cite{Ljung1999,Isermann2011}. In many important applications, the underlying system exhibits nonlinear dynamic behaviour, which is assumed to be adequately captured by a time-indexed data record of system inputs and outputs. The system identification problem is then to obtain a suitable model that best explains this data. Such a model can then be employed as a surrogate for the actual system of interest for the purposes of prediction, control, decision making and analysis.
	
	Due to the enormous variety of systems, there is now a wide array of model classes. These range from linear models~\cite{Ljung1999}, to block-structured nonlinear models~\cite{Wills2013}, through to highly flexible nonlinear black-box models~\cite{Ljung2010}. This paper focuses on probabilistic nonlinear state-space models, a highly flexible class of models that relies on a so-called state, the time evolution of which models the system dynamics. It is important to highlight that the state is typically unknown and, as such, is sometimes called hidden or latent. It is also well recognised that the system identification problem for these models is important and difficult\cite{Ljung2010,Ninness2009}; the states being unknown is a significant source of this difficulty \citep{Schoen2015}.
	
	In this paper, it is assumed that the structure of the state-space model is known but that the associated model parameters are unknown and are to be estimated from the available data. There are many approaches available for estimating these parameters, but we focus on maximum likelihood (ML) parameter estimation. The ML estimate has desirable statistical properties \citep{Aastroem1980} and is commonly used in system identification \cite{Ljung1999,Ninness2009}. However, computing the ML estimate is generally intractable for nonlinear state-space models. This intractability leads to various approximations of the ML estimate; see, e.g. \cite{Schoen2011,Ljung1999,Ninness2009}. These approaches can be broadly grouped into two categories according to the strategy used to handle unknown states; the so-called `marginalisation' and `data augmentation' approaches \citep{Schoen2015}.
	
	Marginalisation approaches compute the likelihood recursively by marginalising the unknown state and then maximise the resultant likelihood directly over the model parameters \cite{Saerkkae2013,Schoen2011}. This likelihood calculation is generally intractable due to the required marginalisation step over an unknown state distribution. This fundamental difficulty has lead to various approximations; for example, \citep{Kokkala2015} and \citep{Schoen2011} consider using an unscented Kalman filter and a particle filter, respectively. While promising results using stochastic quasi-Newton optimisation are shown in \cite{Wills2021}, the marginalisation approach commonly has difficulties regarding undesirable local minima in the resulting optimisation problems leading to issues regarding robustness to the initial parameter estimate \cite{Ljung2006,Ljung2019,Ninness2009}.
	
	Data augmentation approaches treat the unknown state as an auxiliary variable that is estimated alongside the model parameters. A commonly used data augmentation method is the expectation maximisation (EM) approach \cite{Dempster1977}, which iterates between estimating the state and updating the parameters. Obtaining the state estimate for EM is also generally intractable, and instead, various approximations have been used \citep{Chitralekha2009}. Both particle (PSEM) \citep{Schoen2011} and assumed Gaussian \cite{Gasperin2011,Kokkala2014} approximations have been used within EM to deliver approximate ML estimates. Stochastic approximation EM (SAEM)  \citep{Delyon1999} offers improved performance over PSEM for nonlinear state-space models~\cite{Lindsten2013}, primarily due to more efficient iterations that use particles from previous iterations \citep{Schoen2015}.
	
	As an alternative to these methods, the primary contribution of this paper is the presentation of a variational inference (VI) \cite{Jordan1999,Blei2017,Beal2003,Beal2003a} based approach to approximate the ML parameter estimate for nonlinear state-space models. The developed approach is a data augmentation method that provides parameter estimates that \emph{approximate} the maximum likelihood solution. The provided parameter estimate is obtained by solving a single optimisation problem of a standard form with readily available exact first- and second-order derivatives. The presented VI approach jointly estimates the state and parameters and exhibits rapid convergence while remaining robust. This behaviour contrasts with EM methods, where robustness is often observed, but the convergence rate decreases as the parameter estimate approaches the ML estimate. 
	
	Variational methods have also been applied to identify models of different classes in several works; see, for example, \citep{Beal2003,Beal2003a,Risuleo2017,Lindfors2020}. However, these methods do not apply to nonlinear state-space models and, along with many other parameter estimation methods that use VI, are all variational-EM approaches \citep{Blei2017,Wong2020}. Compared to EM approaches, a \emph{key innovation} of the proposed approach in this paper is the simultaneous optimisation over both the state and the model parameters, which appears to offer \emph{significantly} improved convergence rates.

\section{Problem Formulation} \label{ML;sec:problem_formulation}
	The problem of system identification for nonlinear state-space models considered in this paper consists of using the input and output measurements \(\uAll \triangleq u_{1:T} \triangleq \uAllExpand\) and \(\yAll \triangleq y_{1:T} \triangleq \yAllExpand\), respectively, to compute an estimate of the model parameters \(\theta \in \mathcal{R}^{n_\theta}\) for the following model structure
	\begin{subequations} \label{ML:eq:general_model}
		\begin{align}
			x_{k+1} &= f_k\left(x_k,u_k,\pn_k,\theta\right), \\
			y_k &= h_k\left(x_k,u_k,\mn_k,\theta\right),
		\end{align}
	\end{subequations}
	where \(x_k \in \mathcal{R}^{n_x}\) is the state, \( y_k \in \mathcal{R}^{n_y}\) is the observed measurement, \(u_k \in \mathcal{R}^{n_u}\) is the measured input, and the functions \(f_k(\cdot)\) and \(h_k(\cdot)\) describe the process and measurement models, respectively. The process and measurement noise, \(\pn_k\) and \(\mn_k\), respectively, are random variables assumed to belong to distributions of a known form. Parameters of these distributions are included within $\theta$. For ease of exposition, we henceforth drop the explicit dependence on the inputs $\uAll$.
	
	Throughout this paper, we will frequently employ the alternative probabilistic representation of \eqref{ML:eq:general_model}, given by
	\begin{subequations} \label{ML:eq:general_probablisitc_model}
		\begin{align}
			x_{k+1} &\sim p_\theta(x_{k+1} \mid x_k), \\
			y_k &\sim p_\theta(y_k \mid x_k),
		\end{align}
	\end{subequations}
	where \(p\) denotes a probability density function.
	
	In this paper, the problem considered is approximating the maximum likelihood parameter estimate, given by
	\begin{align} \label{eq:ML estimate}
		\theta_{\text{ML}} &= \arg \max_{\theta} \quad \log p_{\theta}\left(\yAll \right).
	\end{align}
	A well-known fundamental difficulty in solving \eqref{eq:ML estimate} is that the log-likelihood function cannot be exactly computed, which extends to the computation of gradient information, and therefore optimisation is challenging. In the following section, we address this challenge by employing a so-called variational inference approach.


\section{Variational Nonlinear System Identification} \label{sec:VI_sys_ID}
	In this section, the use of variational inference applied to nonlinear state-space systems will be presented in Section~\ref{sec:variational inference}. The relationship between variational inference and expectation maximisation is then examined in Section~\ref{sec:link_to_EM}.
 
\subsection{Variational Inference}\label{sec:variational inference} 
	Variational inference is a widely used method where the primary aim is to approximate intractable distributions by a tractable parametric density of an assumed form. The use of the epithet `variational' stems from the idea that these methods rely on optimisation as the primary tool for choosing the member of the parametric assumed density that best matches the intractable distribution of interest. In this paper, we will employ the Kullback-Leibler (KL) divergence~\cite{Kullback1951} as the cost function for determining	optimality. For ease of reference, the KL divergence between two densities $p(z)$ and $q(z)$ is defined as
	\begin{align}
		\label{eq:kl_def}
		\text{KL}[ p(z) \mid \mid q(z)] \triangleq \int p(z) \log
		\frac{p(z)}{q(z)} \text{d}z.
	\end{align}
	In the context of this paper, we motivate the use of VI by noting that the log-likelihood can be expressed using KL divergence via
	\begin{subequations}
		\begin{align}
			\log p_\theta(\yAll) &= \int p_\theta(\xAll \mid \yAll) \log p_\theta(\yAll) \text{d} \xAll \\
								 &= -\int p_\theta(\xAll \mid \yAll) \log \frac{p_\theta(\xAll \mid \yAll)}{p_\theta(\xAll , \yAll)} \text{d} \xAll \\
								 &= -\text{KL}[ p_\theta(\xAll \mid \yAll)  \mid\mid p_\theta(\xAll , \yAll )],
		\end{align}
	\end{subequations}
	where \(\xAll \triangleq x_{1:T+1} \triangleq \xAllExpand\) and the first equality comes from the fact that \(\log p_\theta(\yAll)\) does not depend on $\xAll$ and is, therefore, invariant to expectation relative to any density in $\xAll$. The second equality stems from the fact that $-\log p_\theta(\yAll)= \log p_\theta(\xAll \mid \yAll) - \log p_\theta(\xAll, \yAll)$ and the third equality from the definition of KL divergence \eqref{eq:kl_def}. Therefore, the maximum likelihood problem \eqref{eq:ML estimate} can be equivalently stated as
	\begin{align}
		\label{eq:ml2}
		\theta_{\text{ML}} &= \arg \min_{\theta} \quad \text{KL}[ p_\theta(\xAll \mid \yAll)  \mid\mid p_\theta(\xAll , \yAll )].
	\end{align}
	Unfortunately, the smoothed state distribution $p_\theta(\xAll \mid \yAll)$ cannot be expressed in closed form, rendering the associated optimisation problem \eqref{eq:ml2} intractable.

	In light of this, we propose to replace the intractable smoothed distribution $p_\theta(\xAll \mid \yAll)$ with an $\eta$-parameterised distribution $q_\eta(\xAll)$, called the assumed density, and solve a similar problem
	\begin{align}
		\label{eq:vi_main}
		\theta^\star, \eta^\star &= \arg \min_{\theta,\eta} \quad \text{KL}[ q_\eta(\xAll )  \mid\mid p_\theta(\xAll , \yAll )].
	\end{align}
	Note that the notation introduced here for an assumed density is general and is used with other parameters than just \(\eta\), and for distributions over other variables than just \(\xAll\), within this paper.
	 
	A major benefit of this approach is that $q_\eta(\xAll)$ can be chosen in a convenient manner such that the resulting problem is tractable. This procedure of approximating an intractable density with an assumed density and minimising the KL divergence is known as variational inference (see, e.g. \cite{Blei2017}). The utility of this approach is then highly dependent on the choice of assumed density. In principle, this assumed density is selected to achieve two sometimes competing goals, 1) the assumed density family should be flexible enough to `closely' match $p_\theta(\xAll \mid \yAll)$, and 2) the assumed density $q_\eta(\xAll)$ should result in a tractable optimisation problem \eqref{eq:vi_main}.

	To further discuss this approach, it will be convenient to introduce a cost function
	\begin{align}
		\label{eq:vi_cost}
		\mathcal{L}(\eta, \theta) &= -\text{KL}[ q_\eta(\xAll )  \mid\mid p_\theta(\xAll , \yAll )],
	\end{align}
	so that \eqref{eq:vi_main} becomes $\theta^\star, \eta^\star = \arg \max_{\theta,\eta} \ \mathcal{L}(\eta,\theta)$. The following lemma establishes that the log-likelihood $\log p_\theta(\yAll)$ is bounded below by $\mathcal{L}(\eta,\theta)$, and the gap between the lower bound and log-likelihood is provided by $\text{KL}[ q_\eta(\xAll )  \mid\mid p_\theta(\xAll \mid \yAll )]$. Therefore, as the KL divergence between $ q_\eta(\xAll )$ and $p_\theta(\xAll \mid \yAll )$ diminishes, the lower bound becomes tight.
	\begin{lemma} \label{lem:lower bound}
		The log-likelihood can be expressed as
		\begin{align}
			\label{eq:ml_via_kl}
			\log p_\theta(\yAll) &= \mathcal{L}(\eta,\theta) + \textnormal{KL}[ q_\eta(\xAll )  \mid\mid p_\theta(\xAll \mid \yAll )],
		\end{align}
		and is bounded below according to
		\begin{align}
			\label{eq:nll_bound}
			\log p_\theta(\yAll) \geq \mathcal{L}(\eta,\theta).
		\end{align}
	\end{lemma}
	\begin{proof}
		Using conditional probability
		\begin{align} \label{eq:log_cond_prob}
			\log p_\theta(\yAll) &= \log p_\theta(\xAll, \yAll) - \log p_\theta(\xAll \mid \yAll ).
		\end{align} 
		Through addition and subtraction of \( \log q_\eta(\xAll) \) to the right-hand side of \eqref{eq:log_cond_prob}, this leads to 
		\begin{align} \label{eq:LL_sum}
			\log p_\theta(\yAll) &= \log \frac{ p_\theta(\xAll, \yAll)  }{ q_\eta(\xAll) } + \log \frac{ q_\eta(\xAll) }{ p_\theta(\xAll \mid \yAll )}.
		\end{align}
		Taking expectation of both sides relative to $q_\eta(\xAll)$ delivers
		\begin{align}
			\log p_\theta(\yAll) = & -\int q_\eta(\xAll) \log \frac{q_\eta(\xAll)}{ p_\theta(\xAll, \yAll)} d\xAll \notag\\
								   &\quad +\int  q_\eta(\xAll) \log \frac{ q_\eta(\xAll) }{ p_\theta(\xAll \mid \yAll )} d\xAll,
		\end{align}
		which, from the definition of KL divergence, leads to the expression in \eqref{eq:ml_via_kl}.
		Since KL divergence is non-negative, then $\text{KL}[ q_\eta(\xAll )  \mid\mid p_\theta(\xAll \mid \yAll )] \geq 0$ and \eqref{eq:nll_bound} follows immediately from \eqref{eq:ml_via_kl}.
	\end{proof}
 

\subsection{Comparison to Expectation Maximisation} \label{sec:link_to_EM}
	There are close connections between solutions based on EM and VI; see, e.g. \cite{Neal1998} and \cite{Tzikas2008}. This section examines some of these similarities and differences in the context of system identification for nonlinear state-space models. Towards this end, the EM method can be summarised as iterating
 	\begin{align}
		\label{eq:emits}
		\theta_{k+1} = \arg \min_\theta \quad \text{KL}[ p_{\theta_k}(\xAll \mid \yAll)  \mid\mid p_\theta(\xAll , \yAll )].
	\end{align}
	It is interesting to compare this with the ML problem in \eqref{eq:ml2}, where the only difference for EM is that the smoothed density $p_{\theta_k}(\xAll \mid \yAll)$ is fixed at the $k$'th parameter values. Importantly, EM relies on the smoothed density, which is generally intractable. In theory, the EM approach generates parameters estimates \(\theta_{k}\) that monotonically increases \( \log p_{\theta_k}(\yAll)\) \cite{Cappe2007}.

	The following lemma reveals that EM can be viewed as block-ascent of the VI cost $\mathcal{L}(\eta,\theta)$ over $\eta$ and then $\theta$, respectively. This requires an assumption on $q_\eta(\xAll)$, which essentially means that it is possible to match the smoothed density $p_{\theta_k}\left(\xAll \mid \yAll\right)$.
	\begin{lemma} \label{lem:minKL}
		Assume there exists an \(\eta\) such that
		\begin{align}
			\label{eq:assminKL}
			q_{\eta}(\xAll) = p_{\theta_k}(\xAll \mid \yAll ).
		\end{align}
		Then the EM iterations can be expressed as
		\begin{subequations}
			\begin{align}
				\eta_k  &=  \arg\max_\eta \quad \mathcal{L}\left(\eta,\theta_k\right), \label{eq:viem1}\\
				\theta_{k+1} &= \arg\max_\theta \quad \mathcal{L}\left(\eta_k,\theta\right). \label{eq:viem2}
			\end{align}
		\end{subequations}
	\end{lemma}
	\begin{proof}
		According to \eqref{eq:ml_via_kl}, we can state \eqref{eq:viem1} as
		\begin{align*}
			\eta_k &= \arg\max_\eta \quad  \log p_{\theta_k}(\yAll) - \text{KL}[ q_\eta(\xAll)  \mid\mid p_{\theta_k}(\xAll \mid \yAll )].
		\end{align*}
		Notice that $\log p_{\theta_k}(\yAll)$ does not depend on $\eta$ and, therefore, the above problem becomes
		\begin{align}
    	    \label{eq:em_new_prob}
            \eta_k &= \arg\min_\eta \quad  \text{KL}[ q_\eta(\xAll)  \mid\mid p_{\theta_k}(\xAll \mid \yAll )].
		\end{align}
		Under assumption \eqref{eq:assminKL}, $\text{KL}[ q_\eta(\xAll)  \mid\mid p_{\theta_k}(\xAll \mid \yAll )]$ has a global minimum value of zero, which is achieved for any \(\eta_k\) that solves \eqref{eq:em_new_prob}. Therefore, $q_{\eta_k}(\xAll) = p_{\theta_k}(\xAll \mid \yAll )$, and \eqref{eq:viem2} coincides with \eqref{eq:emits}.
	\end{proof}

	It is important to recognise that, generally, it is not possible to select a \emph{tractable} $q_{\eta}(\xAll)$ such that \( q_{\eta}(\xAll) = p_{\theta_k}(\xAll \mid \yAll )\) for some $\eta$. Similarly, exactly computing the expectations with respect to the smoothed distribution $p_{\theta_k}\left(\xAll \mid \yAll\right)$ is generally not possible for nonlinear state-space models. As such, exactly performing EM on nonlinear state-space models is generally intractable, which removes the guaranteed non-decrease of the log-likelihood. So-called variational EM is obtained from iterating \eqref{eq:viem1}--\eqref{eq:viem2}, where the assumed density $q_\eta(\xAll)$ will not generally match \(p_{\theta_k}(\xAll \mid \yAll )\) \cite{Blei2017,Tzikas2008}. This approach ensures a monotonic sequence of lower bounds to the log-likelihood.

	Relative to these approaches, a major advantage of the proposed VI approach in \eqref{eq:vi_main} is that both $\eta$ and $\theta$ are jointly optimised; typically, this offers improved convergence rates compared with coordinate descent methods \cite{Nocedal2006a}. While the simulations in Section~\ref{sec:examples} confirm this improved convergence rate, we recognise the limitations of drawing general conclusions from these results.


\section{Assumed Gaussian Distribution} \label{sec:assumed_distrbution_and_tractable_approx}
	As mentioned in Section~\ref{sec:VI_sys_ID}, there are two competing goals when choosing the assumed density $q_\eta(\xAll)$, namely that it should be flexible enough to closely match $p_\theta(\xAll \mid\yAll)$ and that it should result in a manageable optimisation problem \eqref{eq:vi_main}. 

	In light of this, we propose $q_\eta(\xAll)$ to be a multivariate Normal (MVN) distribution 
	\begin{align}
          \label{eq:full_q_aw}
		q_\eta(\xAll) = \mathcal{N} \left ( \xAll \mid \mu, \Sigma \right ), \qquad \eta \triangleq \left(\mu, \Sigma \right),
	\end{align}
	where for clarity of exposition, the mean and covariance have the following structure
	\begin{align}
		\label{eq:mu_Sigma}
        \mu = \bmat{\mu_1\\ \mu_2 \\ \vdots \\ \mu_{T+1}}, \ 
        \Sigma = \bmat{
					\Sigma_{1,1}			& \Sigma_{1,2} 	 			& \cdots & \Sigma_{1,T+1} 		\\
					\Sigma_{1,2}^\Transp 	& \Sigma_{2,2} 	 			& \cdots & \Sigma_{2,T+1}	 	\\
					\vdots 					& 							& \ddots & \vdots 			  	\\ 
					\Sigma_{1,T+1}^\Transp 	& \Sigma_{2,T+1}^\Transp 	& \cdots & \Sigma_{T+1,T+1}      },
	\end{align}
	where each $\mu_k \in \mathbb{R}^{n_x}$ and $\Sigma_{k,j} \in \mathbb{R}^{n_x \times n_x}$. It is difficult to comment on whether or not this proposal will achieve the first goal since $p_\theta(\xAll \mid \yAll)$ is unknown in general. Nevertheless, this assumption aligns with other approaches where the smoothed distribution is assumed to be MVN \cite{Saerkkae2013}.  Note that this assumption on the form of the underlying state \emph{may} limit the performance of the overall system identification approach on some systems; examples may include multi-modal systems. Importantly, the MVN assumption leads to a tractable and deterministic optimisation problem \eqref{eq:vi_main}. This section discusses the relevant details of this assumption and presents a different parameterisation of $\eta$ with a dimension significantly less than that of $\left(\mu, \Sigma \right)$.
	

\subsection{Assumed Density $q_\eta(\xAll)$ Details}
	In this section, we show that it is not necessary to compute the full joint state density $q_\eta(\xAll)$ to compute \(\theta^\star\) under the MVN assumption. Towards this, we begin by splitting $\eta$ into two parts. First, $\alpha$ containing all the mean values and the block-tridiagonal covariance matrices from \eqref{eq:mu_Sigma}, and second, $\gamma$ containing all the remaining non-block-tridiagonal covariance matrices from \eqref{eq:mu_Sigma}. 
	\begin{subequations}
		\label{eq:betas}
		\begin{align}
			\alpha &\triangleq \left ( \mu, \left ( \Sigma_{k,k} \right)_{k=1}^{T+1}, \left ( \Sigma_{k,k+1} \right)_{k=1}^{T}\right ),\label{eq:beta1}\\
			\gamma &\triangleq \left ( \left (\Sigma_{k,j}\right)_{j=k+2}^{T+1}\right)_{k=1}^{T-1}. \label{eq:beta2}
		\end{align}
	\end{subequations}
	The following lemma shows that $\gamma$ does not affect the optimal $\alpha$ parameters. This result means that $\gamma$ can be removed from the problem without impacting \(\theta^\star\). The benefit is that the dimension of $\alpha$ is significantly smaller than $\eta$.
	
	\begin{lemma}\label{lem:beta1}
		Assume that $\eta = ( \alpha, \gamma )$ according to \eqref{eq:betas} and let
		\begin{align}
			\label{eq:vi_red}
			\theta^\star, \alpha^\star    &= \arg \max_{\theta,\alpha} \quad \mathcal{L}_R\left(\alpha, \theta\right),
		\end{align}
		where $\mathcal{L}_R\left(\alpha, \theta \right)$ is defined as
		\begin{align}
			\mathcal{L}_R\left(\alpha, \theta\right) = I_1(\alpha) + I_2(\alpha,\theta) - I_3(\alpha),
		\end{align}
		and
		\begin{align*}
			I_1\left(\alpha\right) &= \int q_{\alpha}(x_1) \log p(x_1)  dx_1, \\
			I_{2}\left(\alpha,\theta\right)  &= \sum_{k=1}^{T} \int q_{\alpha} (x_{k:k+1}) \log p_\theta(x_{k+1}, y_k \mid x_k)  dx_{k:k+1}, \\
			I_3\left(\alpha\right)  &= \sum_{k=1}^{T}\int q_{\alpha}(x_{k:k+1}) \log q_{\alpha} (x_{k:k+1}) dx_{k:k+1} \\
									&\quad - \sum_{k=2}^{T} \int q_{\alpha} (x_{k}) \log q_{\alpha} (x_{k}) dx_{k}.
		\end{align*}
		Then 
		\begin{align} \label{eq:vi_full}
			\theta^\star, \alpha^\star, \gamma^\star &= \arg \max_{\theta,\alpha,\gamma} \quad \mathcal{L}\left((\alpha,\gamma), \theta\right).
		\end{align}
	\end{lemma}
	\begin{proof}				
		See Appendix~\ref{app:proof of reduced vi lemma}.
	\end{proof}
	Therefore, since we are ultimately interested in $\theta^\star$, it suffices to solve \eqref{eq:vi_red}, which has the major benefit of involving a much-reduced parameter space. To solve \eqref{eq:vi_red} using standard gradient-based optimisation methods, it is crucial $\mathcal{L}_R\left(\theta, \alpha\right)$ can be evaluated, along with its gradient and possibly Hessian. For $I_1$ (with a suitable choice for the prior) and $I_3$, these calculations are relatively straightforward since the required expectations have known closed-form solutions. At the same time, these calculations are challenging for $I_2$ since they cannot be computed in closed-form in general.	
	
	To ameliorate this, here we introduce a parameterisation of the joint state $(x_k,x_{k+1})$ distribution that leads to both computationally efficient approximations and straightforward computation of the first- and second-order derivatives. To this end, let $\beta_k$ be defined as the set of parameters
	\begin{align}
		\beta_k =  ( \mu_k, \bar{\mu}_k, A_k, B_k, C_k ),
	\end{align}
	where \( A_k, B_k, C_k \in \mathcal{R}^{n_x \times n_x} \) and \( A_k, C_k \) are upper triangular. Then we can parameterise the assumed joint state $(x_k,x_{k+1})$ density as
	\begin{align}
		\label{eq:q_beta_k}
		q_{\beta_k}\left(x_k,x_{k+1}\right) = \mathcal{N}\left(	\begin{bmatrix}
																	x_k \\
																	x_{k+1}
																\end{bmatrix} ; 
																\begin{bmatrix}
																	\mu_k \\ 
																	\bar{\mu}_k
																\end{bmatrix},
																P_k^{\Transp} P_k \right),
	\end{align}
	where $P_k$ is an upper triangular Cholesky factor
	\begin{align}
		P_k = \begin{bmatrix}
					A_k & B_k \\
					0   & C_k
				\end{bmatrix}.
	\end{align}
	We can collect all the $\beta_k$'s into the tuple
	\begin{align}
		\label{eq:beta_def}
		\beta \triangleq (\beta_1, \beta_2,\ldots,\beta_T ).
	\end{align}

	Note that this formulation is over-parameterised and it will lead to inconsistent results since the marginal distribution for $x_k$ can be computed from either \( q_{\beta_{k-1}}\left(x_{k-1},x_{k}\right) \) or \(q_{\beta_k}\left(x_k,x_{k+1}\right) \). The following lemma introduces a constraint set $\Omega$ that eliminates the unnecessary degrees of freedom in $\beta$.
	\begin{lemma}\label{lem:same_set}
		Assume that $\beta$ is defined by \eqref{eq:beta_def}. Let 
		\begin{align*}
			\Omega  &\triangleq \{\beta \mid	B_k^\Transp B_k + C_k^\Transp C_k = A_{k+1}^\Transp A_{k+1}, \notag \\
					&\qquad \qquad \mu_{k+1} = \bar{\mu}_k, \quad k = 1, \dots, T-1 \}. 
		\end{align*}
		Then, for $\beta \in \Omega$, it follows that 
        \begin{align}
			\label{eq:qa_eq_qb}
			q_\alpha(x_{k:k+1}) = q_{\beta_k}(x_{k:k+1}),
		\end{align}
		whenever 
		\begin{align}
			\label{eq:sig_cond}
			\Sigma_{k,k} = A_k^\Transp A_k, \qquad \Sigma_{k,k+1} = A_k^\Transp B_k.
		\end{align}
	\end{lemma}
	\begin{proof}
		According to \eqref{eq:full_q_aw}, note that $q_\alpha(x_{k:k+1})$ is given by
		\begin{align}
			\mathcal{N}\left( 	\begin{bmatrix}
									x_k \\
									x_{k+1}
								\end{bmatrix} ; 
								\begin{bmatrix}
									\mu_k \\ 
									\mu_{k+1}
								\end{bmatrix},
								\bmat{\Sigma_{k,k}  & \Sigma_{k,k+1}\\ \Sigma_{k+1,k} & \Sigma_{k+1,k+1}}   \right).
		\end{align}
        From \eqref{eq:q_beta_k}, we note that $q_{\beta_k}(x_{k:k+1})$ is
        \begin{align}
			\mathcal{N}\left(	\begin{bmatrix}
									x_k \\
									x_{k+1}
								\end{bmatrix} ; 
								\begin{bmatrix}
									\mu_k \\ 
									\bar{\mu}_{k}
								\end{bmatrix},
								\bmat{	A_k^\Transp A_k & A_k^\Transp B_k \\
          								B_k^\Transp A_k & B_k^\Transp B_k + C_k^\Transp C_k} \right ).
        \end{align}
        If $\beta \in \Omega$, then $\bar{\mu}_k=\mu_{k+1}$ and $B_k^\Transp B_k + C_k^\Transp C_k = A_{k+1}^\Transp A_{k+1}$, and therefore \eqref{eq:qa_eq_qb} follows from \eqref{eq:sig_cond}.
	\end{proof}
	The next section makes use of this parameterisation to approximate the required expectation in $I_2$ in a computationally efficient manner.
	

\subsection{VI Approximation}
	As mentioned in the previous section, it is vital for gradient-based optimisation that the expectation in $I_2$ can be approximated, along with its first- and second-order derivatives. The key difficulty in calculating $I_2$ is the computation of $\int q_{\alpha} (x_{k:k+1}) \log p_\theta(x_{k+1}, y_k \mid x_k)  dx_{k:k+1}$. Employing Lemma~\ref{lem:same_set}, we can write this integral in terms of \(\beta_k\), instead of \(\alpha\), as
	\begin{align*}
		E_k(\beta_k,\theta) \triangleq \int q_{\beta_k} (x_{k:k+1}) \log p_\theta(x_{k+1}, y_k \mid x_k)  dx_{k:k+1}.
	\end{align*}
	Importantly, the choice in parameterisation affords a straightforward approximation of $E_k(\beta_k,\theta)$, denoted as $\hat{E}_k(\beta_k,\theta)$, using Gaussian quadrature (see, e.g. \cite{Julier1997,Wan2000}) via
	\begin{align} \label{eq:approx_integral_E}
		\hat{E}_k(\beta_k,\theta) &\triangleq \sum_{j = 1}^{n_s} w_j  \log p_\theta(\bar{x}_{k+1}^j, y_k \mid x_k^j).
	\end{align}
	In the above, $w_j$ are the so-called weights, and $x_k^j$ and $\bar{x}_{k+1}^j$ are the so-called sigma points. Importantly, the weights are predefined constants, and the sigma points are defined as
	\begin{align}
	  \bmat{x_k^j\\ \bar{x}_{k+1}^j} = \bmat{\mu_k 			& A_k^\Transp & 0 \\
	  										\bar{\mu}_{k} 	& B_k^\Transp & C_k^\Transp}  a_j,
	\end{align}
	where the vectors $a_j$ are also constant. Both the weights $w_j$ and vectors $a_j$ depend on the choice of Gaussian quadrature, but they are nevertheless constant for this choice. For further details on specific quadrature methods and the accuracy of different methods in a similar context, see, for example, \cite{Jia2013,Saerkkae2014,Kokkala2015}. Furthermore, the sigma points being linear combinations of the elements of \(\beta_k\) is critical as it significantly simplifies the calculation of both first- and second-order derivatives used in the optimisation.

	Using this approximation, we can define a tractable approximation to problem \eqref{eq:vi_red} via
	\begin{align}
		\label{eq:approx_optim}
		\hat{\beta}, \hat{\theta} = \arg\max_{\beta,\theta}	\quad  &\hat{\mathcal{L}}_R(\beta,\theta), \quad \text{s.t.} \quad \beta \in \Omega,
	\end{align}
	where
	\begin{align}
		\hat{\mathcal{L}}_R\left(\beta,\theta\right) &= I_1\left(\beta\right) + \hat{I}_2(\beta,\theta) - I_3\left(\beta\right),
	\end{align}
	and 
	\begin{align}
	  \label{eq:approx_integral}
	  \hat{I}_2(\beta,\theta) = \sum_{k=1}^{T} \hat{E}_k(\beta_k,\theta).
	\end{align}
	This constrained optimisation problem is of standard form and can be solved using exact first- and second-order derivatives without any further approximations.

	The assumptions contained within \eqref{eq:approx_optim} are summarised as follows: first, an MVN distribution of the state is assumed. Second, \( \log p_\theta(\bar{x}_{k+1}^j, y_k \mid x_k^j) \) can be evaluated and it is twice continuously differentiable, and third, the use of quadrature to approximate the intractable expectations.
	

\subsection{Additive Gaussian Noise} \label{sec:additive noise} 
	In this section, we consider a simplification for nonlinear systems with additive Gaussian noise. This simplification enables a reduction in the number of optimisation variables, brings numerical benefits, and enables an effective approximation of the Hessian to be formed utilising only first-order derivatives.  We consider a model structure of the form
	\begin{align}
	  \label{eq:NLmodel}
	  \bmat{x_{k+1}\\y_k} &= \bmat{ f(x_k,\thetamodel)\\ h(x_k,\thetamodel)} + \bmat{\pn_k\\ \mn_k}, \quad \bmat{\pn_k\\ \mn_k} \sim \mathcal{N}(0,\Pi),
	\end{align}
	where $\theta = \left( \thetamodel, \Pi \right)$ includes model parameters $\thetamodel$ and noise covariance \(\Pi\). As detailed in the following Lemma~\ref{lem:additive_noise}, the structure of this system allows for system identification to be performed using a more structured, reduced size optimisation problem.

	\begin{lemma} \label{lem:additive_noise}
		For systems in the form of \eqref{eq:NLmodel}, identification can be performed by solving the reduced problem
		\begin{align}
			\label{eq:additive ID}
			\hat{\beta}, \hat{\phi} = \arg \max_{\thetamodel,\beta} \quad &\hat{\mathcal{L}}_{\textnormal{AG}}\left(\thetamodel,\beta\right), \quad \textnormal{s.t.} \quad \beta \in \Omega,
		\end{align}
		where
		\begin{subequations}
			\begin{align}
			  \hat{\mathcal{L}}_{\textnormal{AG}}\left(\thetamodel,\beta \right) &= I_1(\beta) + \hat{I}_{2}^{\textnormal{AG}}(\phi, \beta) - I_3(\beta), \\
			  \hat{I}_{2}^{\textnormal{AG}}(\phi,\beta) 						 &= c + \frac{T}{2}\log | \hat{\Pi}(\phi,\beta)|, \label{eq:additive_I23} \\
			  \hat{\Pi}\left(\thetamodel,\beta\right) &= \frac{1}{T}\sum_{k=1}^{T} \sum_{j=1}^{n_s} w_j \xi^j_k \left(\xi^j_k\right)^\Transp,  \label{eq:sample Pi} \\
			  \xi_k^j &=  \begin{bmatrix}
						  	\bar{x}^j_{k+1} - f({x}_k^j,\phi ) \\
						  	y_k -  h({x}^j_k,\phi)
						  \end{bmatrix}, \label{eq:1}
			\end{align}
		\end{subequations}
		and $c$ is a constant that does not depend on $\phi$ or $\beta$.
	\end{lemma}
	\begin{proof}
		Using \eqref{eq:approx_integral} and \eqref{eq:NLmodel}, we can evaluate $\hat{I}_{2}$ as
		\begin{align}	\label{eq:I_23 additive}
			\hat{I}_{2} &= \frac{T}{2}\log |2\pi \Pi| + \frac{1}{2}\sum_{k=1}^{T} \sum_{j=1}^{n_s} w_j (\xi_k^j)^\Transp \Pi^{-1} (\xi_k^j).
		\end{align}
		We note that $(\xi_k^j)^\Transp \Pi^{-1} (\xi_k^j) = \text{tr} \{ (\xi_k^j)^\Transp \Pi^{-1} (\xi_k^j)\} = \text{tr} \{\Pi^{-1}	(\xi_k^j) (\xi_k^j)^\Transp \} $ since the trace operator is invariant to cyclic permutations. Due to the fact that the trace is also a linear operator, we have that 
		\begin{align*}	
			\hat{I}_{2} &= \frac{T}{2}\log |2\pi \Pi| + \text{tr} \left\{ \Pi^{-1} \frac{1}{2}\sum_{k=1}^{T} \sum_{j=1}^{n_s} w_j (\xi_k^j)^\Transp (\xi_k^j) \right \}.
		\end{align*}
		First-order necessary conditions of optimality require that $\partial \hat{I}_{2}\left(\beta,\theta\right) / \partial \Pi = 0$, which occurs  when \\$\Pi = \frac{1}{T}\sum_{k=1}^{T} \sum_{j=1}^{n_s} w_j \xi^j_k \left(\xi^j_k\right)^\Transp$, hence \eqref{eq:sample Pi}. Substituting \eqref{eq:sample Pi} into \eqref{eq:I_23 additive} yields \eqref{eq:additive_I23}, which completes the proof.
	\end{proof}
	Note that the final estimate of $\Pi$ is subsequently given by evaluating \(\hat{\Pi}\left(\hat{\thetamodel},\hat{\beta}\right)\) using \eqref{eq:sample Pi}. Further, decoupled noise can be accommodated by constraining the off-diagonal block of \(\Pi\) to zero.

	This specialisation possesses several benefits compared with a general approach. First, the number of variables required to be optimised is reduced while still addressing the original underlying problem. Second, using only the Jacobian of \(f(x_k, \thetamodel)\) and \(h(x_k, \thetamodel)\), the exact gradient and a good approximation for the Hessian of \(\hat{\mathcal{L}}_{\text{AG}}\left(\thetamodel,\beta\right) \) can be obtained; this is discussed further in Section~\ref{sec:optimisation}, which focuses on the resulting optimisation problem.


\section{Implementation Details} \label{sec:implmentation}
	In this section, several key points regarding implementation details and the resulting optimisation problems are considered. Section~\ref{sec:initiliastion} examines the initialisation of the resultant optimisation problems, and Section~\ref{sec:optimisation} provides some details regarding the optimisations required.
 
	\subsection{Initialisation} \label{sec:initiliastion}
		Similar to all identification approaches for nonlinear systems, an initial estimate of \(\theta\) must be provided and influences both the run-time required, and potentially, the parameter estimate produced. However, compared to EM-based approaches, the proposed approach also requires an initial estimate of each pairwise joint distribution. Similar to EM identification approaches, a smoother can provide this initialisation.
		
		Alternatively, other approaches to initialise the state distributions can be used. This initialisation does not need to satisfy the constraints and introduces an added level of flexibility to exploit. Generally, using the distributions from a filtering pass has proven both effective and straightforward. The initial distribution can also be selected using problem-specific knowledge; an example is when state estimates can be readily calculated from the measurements. This alternative initialisation is particularly beneficial for the additive noise specialisation (Section~\ref{sec:additive noise}), as it removes the requirement to provide initial estimates for \(\Pi\), which is not well known a priori.
 
		Due to the general nature of the optimisation problems, it is not possible to describe the `best' initialisation scheme or how to ensure that undesirable local minima are avoided.  As such, the numerical examples in Section~\ref{sec:examples} focus on robustness with respect to the initialisation of parameters and use all of the initialisation schemes for the state distributions discussed.
  
	\subsection{Optimisation} \label{sec:optimisation}
  		As previously stated, the resulting optimisation problems are of standard form and can be solved using standard solvers. In the numeric examples, the solvers \texttt{fmincon} \cite{MATLAB2018} and \textsc{Knitro} \cite{Byrd2006} are both used. 

		To perform the optimisation effectively, the exact first- and second-order derivatives are used for the general form of the proposed approach. Due to the assumed Gaussian distributions, closed-form expressions, and first- and second-order derivatives exist for all terms, except for \(\hat{I}_{2}\left(\beta,\theta\right)\). Due to the parametrisation of the optimisation problem, the exact gradient and Hessian of  \(\hat{I}_{2}\left(\beta,\theta\right)\) can be efficiently obtained using automatic differentiation \cite{Griewank2008}; we have used CasADi \cite{Andersson2018} for this purpose. 
  
		To further structure the optimisation problem a copy of \(\theta\) for each time step, denoted \(\theta_k\), and the constraints \(\theta_1 = \theta_2 =, \dots, = \theta_T\), are introduced to result in a sparse block-diagonal Hessian. This structure allows the Hessian to be both efficiently formed and used within optimisation routines with a computational complexity that grows linearly in data length, i.e. \(\mathcal{O}\left(T\right)\).
  
		For the additive Gaussian noise specialisation, the derivatives for \( \hat{I}_{2}^{\text{AG}}(\phi, \beta)\) are calculated using Lemma~\ref{der:lem:logdet hessian}.
		\begin{lemma} \label{der:lem:logdet hessian}
  			The gradient and Hessian of \(\hat{I}_{2}^\text{AG}\left(\thetamodel,\beta\right)\) with respect to \(x\) are given by
			\begin{subequations}
		  		\begin{align}
		  			-\nabla_x \hat{I}_{2}^\text{AG}\left(\thetamodel,\beta\right) &= T J_n^\Transp \textnormal{vec}\left( \bar{\Pi}^{-\frac{\Transp}{2}} e(x)\right), \\
		  			-\nabla^2_{x,x} \hat{I}_{2}^\text{AG}\left(\thetamodel,\beta\right) &= T J_n^\Transp J_n - \frac{T}{2} V^\Transp V + S,
		  		\end{align}
			  	where
			  	\begin{align}
			  		J_n &= \left(I_{T n_s} \otimes \bar{\Pi}^{-\frac{\Transp}{2}}\right)  \frac{\partial \textnormal{vec}\left(e\left(x\right)\right)}{\partial x^\Transp }, \\
			  		V &= \tilde{V} + P_{n_x+n_y}\tilde{V}, \\
			  		\tilde{V} &= \left( \bar{\Pi}^{-\frac{\Transp}{2}}e\left(x\right) \otimes I_{n_x+n_y} \right) J_n, \\
			  		\bar{\Pi} &= e\left(x\right) e\left(x\right)^\Transp, \\
			  		e\left(x\right) &= [ e\left(x\right)_1, \dots, e\left(x\right)_T ], \\
			  		e\left(x\right)_k &= [ \sqrt{w_1}\xi^1_k, \dots, \sqrt{w_{n_s}} \xi^{n_s}_k ].			
			  	\end{align}
		   	\end{subequations}
			The matrix \( P_{n_x+n_y} \) is a square vec-permutation matrix \citep{Henderson1979}, \(S\) contains second-order derivative terms, \(x\) is the variables being optimised, \(J_n^\Transp J_n\) is block diagonal, \(V^\Transp V\) is dense, \(V \in \mathcal{R}^{ (n_x+n_y)^2 \times T n_b}\), \(n_b\) is the dimension of the diagonal blocks, and, as \(e(x)\) is a linear function of the sigma points, both \(J_n\) and \(V\) can be obtained using only the Jacobian of \(f(x_k, \thetamodel)\) and \(h(x_k, \thetamodel)\).
		\end{lemma}
		\begin{proof}
  			These equations can be verified by applying the approach to matrix calculus given in \cite{Magnus2019}, properties of Kronecker products, and extensive algebraic manipulations.
		\end{proof}
		The Hessian of \( -\hat{I}_{2}^{\text{AG}}(\phi, \beta)\) is approximated as 
		\begin{align}
			H_{2} &\approx T J_n^\Transp J_n - \frac{T}{2} V^\Transp V.
		\end{align}
		For large \(T\), the dense term renders both forming and factorising \(H_{2}\) computationally intractable. To avoid this issue, we have developed a custom trust-region optimisation routine that, by exploiting the structure of the Hessian approximation, performs an equivalent factorisation with a computational complexity linear in time with respect to the number of measurements. The developed routine is based on the approaches in \cite{Nocedal2006a,Conn2000,More1983}; the full details are beyond the scope of this paper.
	

\section{Examples} \label{sec:examples}

	In this section, we present three numerical examples. First, a stochastic volatility model (Section~\ref{sec:bitcoin_example}) that does not possess additive Gaussian noise. Second, a simulated example of a differential drive robot (Section~\ref{sec:robot_example}) and third, an inverted pendulum using real data (Section~\ref{sec:pendulum_example}). These latter two examples have additive Gaussian noise and structures arising from first-principles modelling, representative of practical applications. The pendulum example differs in that it compares with a particle, rather than an assumed Gaussian method, and that the system is unstable, a property that can be challenging \citep{Ribeiro2020}. 

	For all examples, unless specified otherwise, the proposed approach refers to the general form given by \eqref{eq:approx_optim}. A third-order unscented transform \cite{Julier1997,Wan2000} has been used for all Gaussian quadrature approximations. All numerical examples were conducted on a laptop with an i7\nobreakdash-7820HK processor and 32GB of memory.
 

\subsection{Stochastic Volatility Model} \label{sec:bitcoin_example}
	In this example, we consider the estimation of \(\theta = [a, b, c]\) for the following stochastic volatility model
	\begin{align}
		x_{k+1} &= a + b x_k + \sqrt{c} \pn_k, \qquad 
		y_k  = \sqrt{e^{x_k}}\mn_k,
	\end{align}
	where $\pn_k \sim \mathcal{N}(0,1)$ and $\mn_k \sim \mathcal{N}(0,1)$ is considered using 726 simulated measurements. Results obtained using stochastic approximation EM (SAEM) with a conditional particle filter (CPF) \cite{Lindsten2013}, referred to as CPF-SAEM, which is an asymptotically convergent method, are used to compare against.

	For the proposed method, each joint state distribution is initialised with a mean of $[2, 2]^\Transp$ and a diagonal covariance with standard deviations of \num{0.1}. An initial parameter estimate of $ \theta = [0, 0.5, 1]^\Transp $ is used for both approaches, and 50 particles were used for the CPF-SAEM approach. 

	The proposed method converged in 19 iterations and \SI{3.18}{\second} using \texttt{fmincon} \cite{MATLAB2018} to perform the optimisation. In contrast, CPF-SAEM did not converge and is limited to \num{1000} iterations, which required \SI{326}{\second}  to complete. Fig.~\ref{fig:bitcoin_VI_vs_SAEM} shows the parameter trajectory of both methods and the true values. This result highlights that both methods provide similar estimates that are close to the true parameter values. CPF-SAEM, however, requires a significantly larger quantity of iterations and has only approached the parameter values obtained using the proposed method.
	\begin{figure}[!t]
		\centering
		\includegraphics{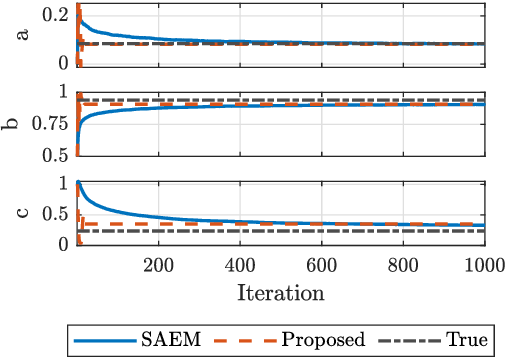}
		\caption{Parameter trajectory vs.~iteration count. The lines for the proposed method have been extended beyond 19 iterations for clarity.}
		\label{fig:bitcoin_VI_vs_SAEM}
	\end{figure}

	To examine if, at least for this example, the parameter estimates produced using the developed method converge as expected for a maximum likelihood method, we have additionally estimated \(\theta\) for \num{50} different data realisations for differing numbers of simulated measurements. Table~\ref{tab:sv_different_lengths} shows the results of this experiment and illustrates an asymptotic trend of decreasing bias and variance of the estimated parameters as the number of measurement samples increases.
	\begin{table}[!t]
		\centering
		\caption{Mean and standard deviation of the estimated parameters from the true values over 50 different realisations for differing numbers of measurement samples.}\label{tab:sv_different_lengths}
		\begin{tabular}{l S[table-format=.2,round-mode=places,round-precision=2] S[table-format=.2,round-mode=places,round-precision=2] S[scientific-notation = true,round-mode=places,round-precision=2]}
			\toprule
			                                 &                                           \multicolumn{3}{c}{Error of the Estimated Parameter}                                           \\
			\cmidrule(lr){2-4}  
			Samples 						 & \multicolumn{1}{c}{\(a\) (true = \num{0.0848} )} & \multicolumn{1}{c}{\(b\) (true = 0.9393)} & \multicolumn{1}{c}{\(c\) (true = 0.2369)} \\ \midrule
			\num{100}                        & \(\num{3.1893e-01}\pm\num{5.9769e-01}\)          & \(\num{-3.8018e-01}\pm\num{4.2774e-01}\)  & \(\num{4.7631e-01}\pm\num{6.5063e-01}\)   \\
			\num{500}                        & \(\num{2.2835e-02}\pm\num{5.0719e-02}\)          & \(\num{-1.9332e-02}\pm\num{2.7540e-02}\)  & \(\num{4.8509e-02}\pm\num{8.1055e-02}\)   \\
			\num{1000}                       & \(\num{2.3140e-02}\pm\num{4.1241e-02}\)          & \(\num{-1.3968e-02}\pm\num{2.2599e-02}\)  & \(\num{3.7880e-02}\pm\num{6.0960e-02}\)   \\
			\num{5000}                       & \(\num{8.6312e-03}\pm\num{1.2837e-02}\)          & \(\num{-6.7006e-03}\pm\num{7.6825e-03}\)  & \(\num{3.0516e-02}\pm\num{2.3570e-02}\)   \\
			\num{10000}                      & \(\num{7.5905e-03}\pm\num{8.6623e-03}\)          & \(\num{-5.3403e-03}\pm\num{5.0392e-03}\)  & \(\num{2.8634e-02}\pm\num{1.7444e-02}\)   \\
			\num{25000}                      & \(\num{5.6510e-03}\pm\num{5.4514e-03}\)          & \(\num{-4.8744e-03}\pm\num{3.2918e-03}\)  & \(\num{2.4962e-02}\pm\num{1.0994e-02}\)   \\ \bottomrule
		\end{tabular}
	\end{table}

	The robustness to initial estimates of the proposed method has also been examined using 100 random initial parameter estimates sampled from
	\begin{align*}
		a \sim \uniformnormal{-0.5}{0.5}, \quad b \sim \uniformnormal{0}{1.5}, \quad c \sim \uniformnormal{0.25}{2}.
	\end{align*}
	Parameter trajectories for each trial are plotted in Fig.~\ref{fig:bitcoin_real_many_inits} and shows that each initialisation has converged to the same parameter estimate.
	\begin{figure}[!t]
		\centering
		\includegraphics{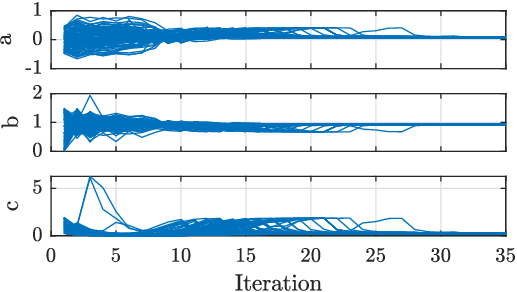}
		\caption{Parameter trajectory vs.~iteration count for 100 random initial estimates using the proposed method.}
		\label{fig:bitcoin_real_many_inits}
	\end{figure}

	In this section, the proposed system identification approach is applied to a system without additive Gaussian noise and compared favourably with CPF-SAEM. The lack of additive Gaussian noise required no alterations to the proposed approach, which proved to be both effective and robust to initial parameter estimates.
 

\subsection{Differential Drive Robot} \label{sec:robot_example}
	In this example, we consider a continuous-time model of a differential drive robot given by
	\begin{align*}
		\!\!
		\begin{bmatrix}
			\dot{q}_1(t) \\
			\dot{q}_2(t) \\
			\dot{q}_3(t) \\
			\dot{p}_1(t) \\
			\dot{p}_2(t)
		\end{bmatrix} \!\!=\!\! \begin{bmatrix}
									\frac{\cos\left( q_3(t)\right) p_1(t)}{m} \\
									\frac{\sin\left( q_3(t)\right) p_1(t)}{m} \\
									\frac{p_2(t)}{J+ml^2} \\
									\frac{-r_1 p_1(t)}{m} - \frac{ml p_2^2(t)}{\left(J+ml^2\right)^2} + u_1(t) + u_2(t) \\
									\frac{\left(lp_1(t)-r_2\right)p_2(t)}{J+ml^2} + a u_1(t) - a u_2(t)
								\end{bmatrix},
	\end{align*}
	where \(r_1 = 1\), \(r_2 = 1\), \(a = 0.5\), \(m = 5\), \(J = 0.2\), \(l = 0.15\), \(u_1(t)\) is the force applied to the left wheel, \(u_2(t)\) is the force applied to the right wheel, and the state vector \(x(t) = \left[{q}_1(t), {q}_2(t), {q}_3(t), {p}_1(t), {p}_2(t)\right]^\Transp\) consists of x-position, y-position, heading, linear momentum, and angular momentum states, respectively. A \SI{50}{\second} simulated trajectory is generated using an ODE solver disturbed by noise sampled from \( \mathcal{N}\left(0,Q\right) \) where \(Q = \text{diag}\left(\num{0.001}, \num{0.001}, \num{1.745e-3}, \num{0.001}, \num{0.001}\right)\) at \SI{0.1}{\second} intervals. Measurements at each interval are obtained according to
	\begin{align*}
		y_k &= \left[{q}_1(t), {q}_2(t), {q}_3(t)\right]^\Transp + \mn_k, \quad  \mn_k \sim \mathcal{N}\left(0,R\right),
	\end{align*} 
	where \( R = \text{diag}\left(0.1^2, 0.1^2, 0.0349^2\right)\). As the system identification methods considered in this paper are all for discrete-time models, a Euler discretisation of the continuous-time dynamics over each \SI{0.1}{\second} interval was used to obtain a model in the form of \eqref{eq:NLmodel}.
	
	\subsubsection{Convergence}
		This section examines the convergence when estimating \(m\), \(l\), \(r\), and a decoupled additive Gaussian covariance. For the proposed method, \textsc{Knitro} \cite{Byrd2006} is used to perform the optimisation. The results are compared against two assumed Gaussian EM approaches that approximate the smoothing step using the URTSS and the VI smoother in \cite{Courts2020}; these are denoted URTSS-EM and VI-EM, respectively and differ in how the assumed Gaussian smoothed density is obtained. From Section~\ref{sec:link_to_EM}, the VI-EM corresponds to variational EM. As such, VI-EM is expected to converge towards the same values as the proposed method, albeit at a slower rate. For each method an initial parameter estimate of \(m = 10\), \(J = 4\), \(l = 0.3\), and
		\begin{align*}
			\Pi = \text{diag}\big( \num{0.01}^2, \num{0.01}^2, \num{0.0035}^2, \num{0.01}^2, \num{0.01}^2, \num{0.1}^2, \num{0.1}^2, \num{0.0349}^2 \big),
		\end{align*}
		is used. For the proposed method, the state distributions have been initialised using the filter developed in \cite{Courts2020}.

		The proposed method converged to a locally optimal solution in \num{194} iterations in \SI{297}{\second}; neither the URTSS EM nor the VI EM approaches converged within the \num{20e3} iteration limit, which required \SI{1.07}{\hour} and \SI{5.81}{\hour}, respectively, to reach. Fig.~\ref{fig:robot_ID_comparision} shows the cost obtained using each method as a function of iteration count and illustrates the \num{194} iterations of the proposed method outperformed the \num{20e3} iterations of both the EM approaches. As expected, VI-EM asymptotically approaches the cost achieved by the proposed method. Contrarily, URTSS-EM does not approach this cost; neither does it maintain the desired monotonic behaviour of EM. This highlights the benefit of following the more principled approach regarding any approximations introduced.
		\begin{figure}[!t]
			\centering
			\includegraphics{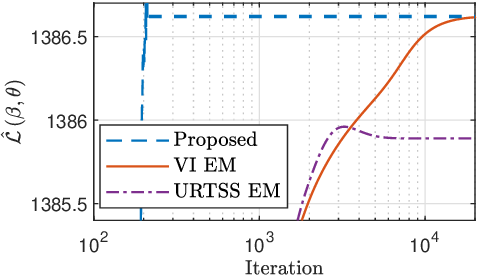}
			\caption{Comparison of proposed, VI-EM, and URTSS-EM approaches to identification on a simulated robot example. Proposed extended past the \num{194} iterations required for convergence for clarity.}
			\label{fig:robot_ID_comparision}
		\end{figure}

	\subsubsection{Robustness to Initialisation}
		This section examines the robustness of the proposed method to the initial parameter estimate. For this, estimation of \(m\), \(l\), \(r\), and coupled noise term \(\Pi\) is considered from \num{75} differing parameter initialisations given by
		\begin{gather*}
			m \sim \uniformnormal{0.5}{15}, \quad J \sim \uniformnormal{0.01}{10}, \quad l \sim \uniformnormal{0.01}{0.5}, \\
			\Pi = \text{diag}\big( \num{0.1}^2, \num{0.1}^2, \num{0.0349}^2, \num{0.1}^2, \num{0.1}^2, \num{0.5}^2, \num{0.5}^2, \num{0.1745}^2 \big).
		\end{gather*}
		The initial state distributions for each optimisation were obtained by running the URTSS from the initial parameter estimates.

		Table~\ref{tab:robot_robustness} shows a subset of the estimated parameters for this experiment, where \(\Pi_{ii}\) indicates the $i^{\textnormal{th}}$ diagonal element of \(\Pi\). In the third column, the mean parameter estimate for each of the \num{75} trials is shown and is an effective estimate of the true parameters. The fourth column shows the maximum difference between this mean and each trial and indicates that each of the \num{75} trials converged to the same parameter estimate, subject to a small tolerance.

		These results indicate the robustness of the proposed approach and highlights that, at least for this example, the proposed method does not suffer from a potentially large quantity of undesirable local minima. The estimates for the off-diagonal elements of \(\Pi\) performed similarly as the parameters shown in Table~\ref{tab:robot_robustness}. Due to the multi-variable nature of \(\Pi\), these numerical values are excluded for brevity.
		\begin{table}[!t]
			\centering
			\caption{Parameter estimate for 75 differing initial parameter estimates to illustrate robustness.}\label{tab:robot_robustness}
			\begin{tabular}{l S[table-format=.2,round-mode=places,round-precision=2] S[table-format=.2,round-mode=places,round-precision=2] S[scientific-notation = true,round-mode=places,round-precision=2]}
				\toprule
				Parameter    & \multicolumn{1}{l}{True} & \multicolumn{1}{l}{Estimated} & \multicolumn{1}{l}{Maximum difference} \\ \midrule
				\(m\)        & \num{5}                  & \num{5.0727}                  & \num{6.1518e-9}                        \\
				\(J\)        & \num{2}                  & \num{2.2353}                  & \num{1.2048e-8}                        \\
				\(l\)        & \num{0.15}               & \num{0.1435}                  & \num{6.6374e-10}                       \\
				\(\Pi_{11}\) & \num{1.00e-3}            & \num{1.1154e-3}               & \num{5.7510e-9}                        \\
				\(\Pi_{22}\) & \num{1.00e-3}            & \num{1.4088e-3}               & \num{3.0441e-8}                        \\
				\(\Pi_{33}\) & \num{1.7435e-3}          & \num{1.8535e-3}               & \num{3.7630e-9}                        \\
				\(\Pi_{44}\) & \num{1.00e-3}            & \num{4.8965e-3}               & \num{1.5193e-7}                        \\
				\(\Pi_{55}\) & \num{1.00e-3}            & \num{1.4029e-3}               & \num{1.1498e-8}                        \\
				\(\Pi_{66}\) & \num{1.00e-2}            & \num{1.0022e-2}               & \num{4.6894e-7}                        \\
				\(\Pi_{77}\) & \num{1.00e-2}            & \num{8.2476e-3}               & \num{4.1427e-7}                        \\
				\(\Pi_{88}\) & \num{1.2185e-3}          & \num{1.3841e-3}               & \num{1.3314e-7}                        \\ \bottomrule
			\end{tabular}
		\end{table}
 

\subsection{Inverted Pendulum} \label{sec:pendulum_example}
	This section considers a rotational inverted pendulum, or Furata pendulum \cite{Furuta1992}, using data collected from a Quanser QUBE-Servo 2. Letting the state vector used to model the Furata pendulum be $x = \pmat{\vartheta & \alpha & \dot{\vartheta} & \dot{\alpha}}^{\Transp}$, where $\vartheta$ and $\alpha$ are the base arm and pendulum angles, respectively, and the controllable input to the system is the motor voltage $V_{m}$.  Then, the continuous time dynamics are then given by
	\small
	\begin{equation*}\label{eq:pend_process}
		\begin{split}
			&M(\alpha)\pmat{\ddot{\vartheta} \\ \ddot{\alpha}} + \nu(\dot{\vartheta},\dot{\alpha})\pmat{\dot{\vartheta} \\ \dot{\alpha}} = \pmat{\frac{k_m(V_m - k_m\dot{\vartheta})}{R_m} - D_r \dot{\vartheta} \\ -\frac{1}{2}m_pL_pg\sin(\alpha)-D_p\dot{\alpha}},\\
			&M(\alpha) \hspace{-0.75mm}=\hspace{-0.75mm} \pmat{m_pL_r^2\hspace{-0.5mm}+\hspace{-0.5mm}\frac{1}{4}m_pL_p^2(1\hspace{-0.5mm}-\hspace{-0.5mm}\cos(\alpha)^2)\hspace{-0.5mm}+\hspace{-0.5mm}J_r & \hspace{-2mm} \frac{1}{2}m_pL_pL_r\cos(\alpha) \\ \frac{1}{2}m_pL_pL_r\cos(\alpha) & J_p + \frac{1}{4}m_pL_p^2}, \\
			&\nu(\dot{\vartheta},\dot{\alpha}) \hspace{-0.75mm}=\hspace{-0.75mm} \pmat{\frac{1}{2}m_pL_p^2\sin(\alpha)\cos(\alpha)\dot{\alpha} & \hspace{-2mm} - \frac{1}{2}m_pL_pL_r\sin(\alpha)\dot{\alpha} \\ - \frac{1}{4}m_pL_p^2\cos(\alpha)\sin(\alpha)\dot{\vartheta} & 0}, \\
		\end{split}
	\end{equation*}
	\normalsize 
	where $m_p$ is the pendulum mass, $L_r$, $L_p$ and the rod and pendulum lengths, $J_r$, $J_p$ are the rod and pendulum inertias, $R_m$ and $k_m$ are the motor resistance and constant, $D_p$ and $D_r$ are the pendulum and arm damping constants. The considered process model is a two-step Euler discretisation of these continuous-time dynamics over an \SI{8}{\milli\second} sampling time subsequently disturbed by noise \(\pn_k\). The available measurements are from encoders on the rod and pendulum angle, and the motor current. The resulting measurement model is
	\begin{align*}
		y_k &= 	\begin{bmatrix}
					\vartheta & \alpha & \frac{V_m - k_m \dot{\vartheta}}{R_m} 
				\end{bmatrix}^\Transp + \mn_k, 
				\qquad 
				\begin{bmatrix}
					\pn_k^\Transp & \mn_k^\Transp
				\end{bmatrix}^\Transp \sim \mathcal{N}\left(0,\Pi\right),
	\end{align*}
	which, together with the discrete process model obtained using the two-step Euler integration, results in a model in the form of \eqref{eq:NLmodel}.
  	We are interested in estimating the parameters $\theta=\left(D_r, D_p, J_r, J_p, k_m, R_m, \Pi\right)$. 
 
	This section compares the effectiveness and robustness of the additive noise specialisation of the proposed method with CPF-SAEM from \num{50} differing parameter initialisations. These initialisations sampled from
	\begin{align*}
		&J_r \sim \uniformnormal{\num{1e-7}}{0.01}, \quad &&D_r \sim \uniformnormal{\num{1e-7}}{0.01}, \\
		&J_p \sim \uniformnormal{\num{1e-7}}{0.01}, \quad &&D_p \sim \uniformnormal{\num{1e-7}}{0.01}, \\
		&k_m \sim \uniformnormal{0.001}{1}, 		\quad &&R_m\sim \uniformnormal{4}{15},
	\end{align*}
	with the noise covariance term initialised as 
	\begin{align*}
			\Pi = \text{diag}\Big( \num{7.6154e-5}, \num{7.6154e-5}, \num{0.0012}, \num{0.0012}, \num{3.0462e-4}, \num{3.0462e-4}, \num{0.01} \Big),
	\end{align*}
	on a single data set consisting of \num{375} measurements, during which the pendulum undergoes full rotations. Note that the noise terms are only utilised for the CPF-SAEM approach; in accordance with Section~\ref{sec:initiliastion} they are not required for the additive noise version of the proposed approach.

	For these trials, the initial position and velocity states used the respective measurements and a finite-difference approximation for the means, and standard deviations of \ang{1} and \SI{10}{\degree\per\second}, respectively. For the proposed method, a relative function tolerance of \num{1e-8} on \(\hat{\mathcal{L}}_{\text{AG}}\left(\thetamodel,\beta \right)\) was the termination condition, which required a median of \num{170} iterations and \SI{112}{\second} to achieve. Importantly, the proposed method proved to be robust to the differing initialisations, with all trials converging to the same value subject to a small numeric tolerance.    

	The CPF-SAEM approach was performed using a fully adapted auxiliary particle filter \cite{Pitt1999} with \num{100} particles, \num{50} particles for the backward simulation smoother \cite{Godsill2004}, and a burn-in of \num{100} iterations before applying the stochastic approximation. The iterations terminated when the change in all parameter estimates between successive iterations fell below \num{1e-3} for five of the last ten iterations, which required a median of \num{156} iterations and \SI{308}{\second} to achieve. In contrast to the proposed method, CPF-SAEM failed on \num{12} of the \num{50} trials.

	Table~\ref{tab:pendulum_robustness} shows the mean and standard deviation of a subset of the estimated parameter values of each successful trial, where \(\Pi_{ii}\) indicates the $i^{\textnormal{th}}$ diagonal element of \(\Pi\). The closeness of the estimates for the proposed method to the successful CPF-SAEM trials illustrates that, despite the approximations introduced, similar parameter estimates are obtained. While, for brevity, the numeric values of off-diagonal elements of \(\Pi\) are omitted here, they were similarity estimated.
	\begin{table}[!t]
		\caption{Mean and standard deviation of the parameter estimates for each successful trial from 50 differing initial values. Twelve of the CPF-SAEM runs were unsuccessful and have been censored.}
		\label{tab:pendulum_robustness}
		\centering
		\begin{tabular}{l S[scientific-notation = true,round-mode=places,round-precision=2] S[scientific-notation = true,round-mode=places,round-precision=2]}
			\toprule
			Parameter    & \multicolumn{1}{c}{Proposed}           & \multicolumn{1}{c}{CPF-SAEM}          \\ \midrule
			\(J_r\)      & \(\num{1.0745e-4}\pm\num{8.1713e-11}\) & \(\num{1.0842e-4}\pm\num{3.2734e-7}\) \\
			\(J_p\)      & \(\num{2.9242e-5}\pm\num{3.3197e-11}\) & \(\num{2.9002e-5}\pm\num{7.7968e-8}\) \\
			\(k_m\)      & \(\num{0.0457}\pm\num{2.7279e-8}\)     & \(\num{0.0499}\pm\num{4.6046e-4}\)    \\
			\(R_m\)      & \(\num{9.5862}\pm\num{4.0597e-7}\)     & \(\num{10.2400}\pm\num{0.0742}\)      \\
			\(D_p\)      & \(\num{4.6821e-5}\pm\num{7.6643e-11}\) & \(\num{5.0327e-5}\pm\num{7.2373e-7}\) \\
			\(D_r\)      & \(\num{2.8108e-4}\pm\num{1.8684e-9}\)  & \(\num{2.1514e-4}\pm\num{4.3191e-6}\) \\
			\(\Pi_{11}\) & \(\num{3.9431e-6}\pm\num{1.5805e-10}\) & \(\num{6.1219e-6}\pm\num{3.4102e-7}\) \\
			\(\Pi_{22}\) & \(\num{1.9859e-6}\pm\num{2.1645e-10}\) & \(\num{4.6300e-6}\pm\num{2.9801e-7}\) \\
			\(\Pi_{33}\) & \(\num{2.7232e-2}\pm\num{2.7143e-7}\)  & \(\num{2.5760e-2}\pm\num{1.0993e-3}\) \\
			\(\Pi_{44}\) & \(\num{3.0969e-2}\pm\num{2.1672e-7}\)  & \(\num{3.0787e-2}\pm\num{1.6126e-3}\) \\
			\(\Pi_{55}\) & \(\num{2.4198e-6}\pm\num{5.2118e-10}\) & \(\num{1.3145e-6}\pm\num{9.4193e-8}\) \\
			\(\Pi_{66}\) & \(\num{1.7495e-6}\pm\num{2.0266e-10}\) & \(\num{1.2776e-6}\pm\num{6.6770e-8}\) \\
			\(\Pi_{77}\) & \(\num{1.9067e-2}\pm\num{1.0211e-9}\)  & \(\num{1.9563e-2}\pm\num{1.1781e-4}\) \\ \bottomrule
		\end{tabular}
	\end{table} 

	This section has demonstrated the applicability of the developed system identification approach on a real unstable system of practical interest. In particular, this example has shown the robustness to initialisation, computational efficiency, and ability to provide parameter estimates close to those of particle methods.


\section{Conclusion} \label{sec:conclusion}

	The contribution of this paper is to present a VI based approach to system identification for nonlinear state-space models. The resulting system identification approach consists of a single, deterministic, optimisation problem. Due to the assumed density and constrained parametrisation chosen, this optimisation problem is of a standard form and is efficiently solved using readily available exact first- and second-order derivatives. A specialisation for systems with additive Gaussian noise was also presented. The proposed method has been numerically examined on a range of simulated and real examples to illustrate the robustness and effectiveness and is compared favourably to state-of-the-art alternatives.	

	\bibliographystyle{plain}       
	\bibliography{ID_paper}
	
	\appendix 	
 	

\section{Proof of Lemma~\ref{lem:beta1}} \label{app:proof of reduced vi lemma}
	\begin{proof}	
		The cost function \(\mathcal{L}\left((\alpha,\gamma), \theta\right)\) is given by
		\begin{align}
			\mathcal{L}\left((\alpha,\gamma), \theta\right) &= \int q_{\eta}(\xAll) \log \frac{ p_\theta(\xAll, \yAll)}{q_\eta(\xAll)} d\xAll \\
			&= \int q_\eta(\xAll) \log p_\theta(\xAll, \yAll) d\xAll \notag \\
			&\quad - \int q_\eta(\xAll) \log  q_\eta(\xAll) d\xAll. \label{eq: full_entropy_integral}
		\end{align}	
		Due to the Markovian nature of state-space models, conditional probability, and that \(q_{\eta}(\xAll)\) is a probability distribution, similar to \cite{Schoen2011,Chitralekha2009,Saerkkae2013,Vrettas2008,Courts2020}, we can write
		\begin{align}
			\int q_\eta(\xAll) \log p_\theta(\xAll, \yAll) d\xAll = I_1\left(\alpha\right) + I_{2}\left(\alpha,\theta\right).
		\end{align}
		This is as only the pairwise joint distributions are required in the calculation of the first integral of \eqref{eq: full_entropy_integral}, i.e. it is independent of \(\gamma\).
		
		As shown in \cite{Courts2020}, the optimal assumed density \(q_{\eta^\star}\left(\xAll\right)\) factors according to
		\begin{align} \label{eq:q_factorisation}
			q_{\eta^\star}\left(\xAll\right) = q_{\eta^\star}\left(x_1\right) \prod_{k=1}^{T} q_{\eta^\star}\left(x_{k+1} \mid x_k\right),
		\end{align}
		and that \(\gamma^\star = g( \alpha^\star )\), where \(g(\alpha)\) is a function (defined by the process detailed in Appendix~C of \citep{Courts2020}) that provides a \(\gamma\) such that the above factorisation is satisfied. As such, by substituting \(\gamma = g(\alpha) \), the second integral in \eqref{eq: full_entropy_integral} is given by
		\begin{align}
			\int q_\eta(\xAll) \log  q_\eta(\xAll) d\xAll = I_3\left(\alpha\right).
		\end{align}
		Therefore, we can equivalently find \(\theta^\star\) and \(\alpha^\star\) from either \eqref{eq:vi_red} or \eqref{eq:vi_full}.
	\end{proof}
 	
\end{document}